\documentclass[12pt, journal,draftclsnofoot, onecolumn]{IEEEtran}

\usepackage[T1]{fontenc}%
\ifCLASSINFOpdf
  \usepackage[pdftex]{graphicx}
\else
  \usepackage[dvips]{graphicx}

\fi

\usepackage{array}

\ifCLASSOPTIONcompsoc
  \usepackage[caption=false,font=normalsize,labelfont=sf,textfont=sf]{subfig}
\else
  \usepackage[caption=false,font=footnotesize]{subfig}
\fi

\usepackage{amsmath}
\usepackage{amsthm}
\usepackage{amssymb} 
\usepackage{bbm}
\usepackage{cite}

\usepackage{url}

\usepackage[dvipsnames]{xcolor}
\usepackage{comment}

\usepackage[]{algorithm}
\usepackage[noend]{algorithmic}
\usepackage{enumitem}
\usepackage[]{epstopdf}

\graphicspath{{figures/}}
\newtheorem{lemma}{Lemma}

\DeclareMathOperator*{\argmax}{arg\,max}

\allowdisplaybreaks

\begin{document}
%
\markboth{Submitted to IEEE Transactions on Communications}{}

\title{Teacher-Student Learning based Low Complexity Relay Selection in Wireless Powered Communications}

%
%
%

\author{ Aysun Gurur Onalan,~\IEEEmembership{Graduate Student Member,~IEEE}, Berkay~Kopru, Sinem~Coleri,~\IEEEmembership{Fellow,~IEEE} \thanks{A. G. Onalan and S. Coleri are with the Department
of Electrical and Electronics Engineering, and B.Kopru is with Department of Computer Engineering, Koc University, 34450, Istanbul, Turkey e-mail: \{aonalan17, bkopru17, scoleri\}@ku.edu.tr. \\ \indent Sinem Coleri acknowledges the support of the Scientific and Technological Research Council of Turkey 2247-A National Leaders Research Grant \#121C314.
}}

\maketitle

\begin{abstract}

Radio Frequency Energy Harvesting (RF-EH) networks are key enablers of massive Internet-of-things by providing controllable and long-distance energy transfer to energy-limited devices. Relays, helping either energy or information transfer, have been demonstrated to significantly improve the performance of these networks. This paper studies the joint relay selection, scheduling, and power control problem in multiple-source-multiple-relay RF-EH networks under nonlinear EH conditions. We first obtain the optimal solution to the scheduling and power control problem for the given relay selection. Then, the relay selection problem is formulated as a classification problem, for which two convolutional neural network (CNN) based architectures are proposed. While the first architecture employs conventional 2D convolution blocks and benefits from skip connections between layers; the second architecture replaces them with inception blocks, to decrease trainable parameter size without sacrificing accuracy for memory-constraint applications. To decrease the runtime complexity further, teacher-student learning is employed such that the teacher network is larger, and the student is a smaller size CNN-based architecture distilling the teacher's knowledge. A novel dichotomous search-based algorithm is employed to determine the best architecture for the student network.  Our simulation results demonstrate that the proposed solutions provide lower complexity than the state-of-art iterative approaches without compromising optimality.

\end{abstract}


\section{Introduction}

Energy Harvesting (EH) is a promising alternative to fixed power supplies (e.g., batteries) for energy-constrained Internet-of-Things (IoT) applications, such as smart cities and wearable devices \cite{Ma2020}, by enabling the net-zero energy objective of sixth generation (6G) mobile networks \cite{Viswanathan2020}.  The use cases and feasibility of wireless-powered IoT have been investigated in many standardization bodies, including IEEE 802.11 \cite{80211amp} and 3GPP\cite{3gpp}. Radio Frequency (RF) signals are preferred over other energy sources such as vibration and sun due to their independence from climate conditions, accessibility, and more predictable nature \cite{ashraf_simultaneous_2021}. RF-EH network design follows two main protocols: Simultaneous Information and Power Transfer (SWIPT), in which power and information are sent simultaneously, and Wireless Powered Communication Networks (WPCN), in which power and information are sent consecutively in downlink (DL) and uplink (UL) \cite{Lu2015}.  SWIPT and WPCN are studied  for various objective functions such as maximum throughput, minimum schedule length, and energy efficiency. Usage of relays has been demonstrated to improve network performance in relation to these objectives. Selecting the best relay for the information or energy transfer further contributes to improved performance. 



 The relay selection problem is studied in SWIPT networks with both amplify-and-forward \cite{gupta_time-switching_2019} and decode-and-forward relays \cite{zheng_swipt_2023}. \cite{gupta_time-switching_2019} jointly optimizes relay selection, subcarrier pairing, user allocation, and power allocation to improve the energy efficiency of multiple-source-multiple-relay multi-carrier networks. An iterative algorithm is proposed based on a tractable quasi-concave form by applying a series of convex transformations. \cite{zheng_swipt_2023} jointly optimizes relay selection, power allocation, beamforming, and signal splitting to maximize the end-to-end throughput of single-source-multiple-relay networks. Again, an iterative algorithm is proposed based on decoupling the non-convex variables into different sub-problems. As iterative algorithms suffer from high runtimes, sub-optimal relay selection schemes, based on channel characteristics of source-to-relay and/or relay-to-destination links, are also studied in both works.

The relay selection problem is investigated in half-duplex \cite{Onalan2020,xie_age_2023} and full-duplex \cite{kazmi2020, kazmi_relay_2021} WPCN for different objectives such as minimum length scheduling \cite{Onalan2020,kazmi2020}, maximum throughput \cite{kazmi_relay_2021}, age-of-information minimization \cite{xie_age_2023}.  \cite{Onalan2020, kazmi2020} and \cite{kazmi_relay_2021} formulate joint relay selection, scheduling, and power allocation problems in multiple-source-multiple-relay networks. They first solve scheduling and power control problem for fixed relay selection, and then, search for the optimum relay selection by using iterative algorithms. \cite{xie_age_2023} considers single-source-multiple-relay IoT system. It involves selecting a fixed number of relays chosen among all those that received the status update from the source at the earliest. Subsequently, it jointly optimizes the number of relays and the packet lengths at the source and relays by using a gradient-based two-step search algorithm. All proposed iterative algorithms that solve a complex mathematical equation in each iteration suffer from high run-time. However, algorithms with high runtime may not be practical for time-critical applications.

 Deep Neural Networks (DNNs) and Deep Reinforcement Learning (DRL) have recently been applied to relay selection problems in EH networks to provide low complexity, high accuracy, and robust solutions \cite{han_joint_2022,nguyen_deep-neural-network-based_2021}.  \cite{han_joint_2022} applies deep deterministic policy gradient (DDPG) and deep Q-network (DQN) algorithms for joint relay selection and power allocation problem with the objective of cumulative throughput maximization in  single-source-multiple-relay networks.  \cite{nguyen_deep-neural-network-based_2021} proposes DNN-based solution for the relay selection problem with the objective of maximum throughput in single-source-multiple-relay SWIPT networks. DNNs are used to predict the achievable throughput, and an iterative algorithm performs relay selection based on these predicted throughputs. These algorithms designed for single-source-multiple-relay networks are not directly applicable to the relay selection problem in multiple source-destination networks since individually optimum relays may not ensure the best relay selection for the entire network. 
 
 Deep learning techniques have been employed for relay selection in multi-source-multi-relay EH networks \cite{wang_multi-featured_2021,nguyen_short-packet_2023,Onalan23}. \cite{wang_multi-featured_2021} proposes a distributed actor-critic reinforcement learning algorithm running separately for each source. The purpose is to optimize relay selection and maximize the total data delivery ratio. This work assumes the arrival of energy packets as a random process and does not consider a dedicated RF energy source.  \cite{nguyen_short-packet_2023} formulates the joint relay selection and beamforming problem to maximize throughput.  Convolutional Neural Networks (CNNs) with multiscale-accumulation connections are proposed for throughput prediction. The relay selection is performed by iterative algorithms based on the predicted throughput. Although the deep learning approaches solving the sub-problems decrease complexity in \cite{ nguyen_short-packet_2023}, the iterative algorithms are still needed to choose relays. In \cite{Onalan23}, a feed-forward DNN is trained for the low-complexity solution to the classification problem in relay selection for minimum length scheduling under linear EH assumption.  This approach  has scalability issues for large network sizes due to the exponentially increasing number of classes with network size. In multi-source-multi-relay WPCN, the relay selection problem with the objective of the minimum length scheduling has not been studied under more realistic non-linear EH models. However, the development of low complexity and scalable solutions for minimum length scheduling under realistic conditions is crucial for time critical and delay sensitive IoT applications.


  Knowledge distillation (KD) techniques are proposed to compress large deep learning models \cite{Tang2021}  or integrate domain-specific knowledge to deep learning frameworks \cite{Zheng2022,Qi2023}. KD involves transferring knowledge from a complex network, referred to as teacher network, to a simpler network, known as the student network. The aim is to improve the performance of the student network, bringing it closer to the performance level of the complex network. Although KD is commonly applied in speech and image processing, its applications in wireless communications are currently limited. \cite{Tang2021} employs KD to provide a low-complexity CSI estimation in MIMO communication. \cite{Zheng2022} applies KD to harness network domain knowledge to improve the robustness of and trust on DRL. \cite{Qi2023} uses KD within federated learning to overcome the heterogeneity of local models to improve modulation classification tasks. None of these studies use KD to propose low-complexity wireless resource allocation and optimization algorithms.


The goal of this paper is to propose a deep learning based solution framework for the joint relay selection, scheduling, and power control problem in multiple-source-multiple-relay WPCN under non-linear EH conditions. Both sources and relays within the network have EH capabilities. By using the harvested energy, each source needs to send information via a selected relay or directly to the AP. Simultaneously, each relay is responsible for conveying all gathered information from sources to the AP. The relay selection problem is formulated as a classification problem, and we propose convolutional neural network (CNN) based architectures. To reduce complexity without compromising accuracy, we suggest using inception blocks instead of convolutional blocks and knowledge distillation techniques. The major contributions of the paper are listed as follows:
\begin{itemize}
\item We formulate the joint relay selection, scheduling, and power control problem in multiple-source-multiple-relay WPCN under non-linear EH conditions, for the first time in the literature. We adapt the formulation based on the linear EH model in \cite{Onalan2020} to non-linear EH conditions. The non-linearity in the energy conversion model changes the required EH duration for the demanded IT, consequently, impacts the optimal solution for all variables, including relay selection, time, and power allocations.   The formulated problem is a non-convex mixed integer non-linear problem (MINLP) and NP-hard. 

\item We propose an optimization theory based solution framework for the joint relay selection, scheduling and power control problem in multiple-source-multiple-relay WPCN under non-linear EH conditions, for the first time in the literature. This framework serves as a benchmark in the performance comparison, and it provides training data for the proposed CNN-based solution strategy. We follow the solution strategy proposed in \cite{Onalan2020}, which starts from a simpler sub-problem of scheduling and power control for a given relay selection, and then searches for the optimal relay selection via branch-and-bound based algorithm. We adapt this algorithm to non-linear EH condition. 

\item We propose a deep learning based solution framework for the relay selection for minimum length scheduling in multiple-source-multiple-relay WPCN, for the first time in the literature. We define the relay selection problem as a classification problem and propose two novel Convolutional Neural Network (CNN) based architectures for low complexity solutions. The inputs of both architectures are channel gains. The outputs are relay selection variables represented by different classes. The first architecture uses convolutional blocks and benefits from skip connections between layers. The second architecture replaces several convolutional blocks with inception blocks to lower the trainable parameter size without sacrificing accuracy for memory-constrained applications. 



\item We propose the usage of KD techniques, specifically teacher-student learning, to decrease the complexity of CNN based solution for the relay selection problem, for the first time in the literature. 
We distill the knowledge of the initial well-established CNN model (teacher network), to a smaller and less complex architecture (student network), by including the soft outputs of the teacher in the training process of the student.

\item To maximize the advantages of teacher-student learning, we propose a novel dichotomous search-based algorithm to determine the optimal student network architecture, achieving a specified level of accuracy with the lowest possible complexity, for the first time in the literature. This search adaptively selects the number of trainable parameters, indicating the complexity of the architecture, of the CNN-based model by evaluating the trade-off between complexity and validation loss in each iteration. 


\item  We illustrate the superior performance of the proposed deep learning frameworks compared to the iterative benchmark algorithms in terms of schedule length and complexity via extensive simulations.

\end{itemize}

The rest of the paper is organized as follows. Section~\ref{section:sysmod} describes the system model. In Section~\ref{section:probForm}, the optimization problem is formulated, and the solution strategy is discussed. Section~\ref{sec:schedule} formulates and solves the scheduling and power control problem for a given relay selection. For the relay selection problem, Section~\ref{sec:relay-learning} presents the proposed CNN-based models and teacher-student learning approach. Section~\ref{section:simRes} presents the simulation results. Finally, Section~\ref{section:conc} concludes the paper.

\section{System Model} \label{section:sysmod}

We consider multiple-source-multiple-relay half-duplex WPCN. Sources send information to the AP in the uplink by exhausting the energy harvested in the downlink, either independently or with the help of a relay. Both sources and relays lack embedded energy supplies, requiring them to replenish energy with the RF signal broadcasted by the AP for IT. The WPCN contains one AP, $N$ EH sources, denoted by $S_i$, $i = 1,2,\ldots, N$, and $K$ EH decode-and-forward (DF) relays, denoted by $R_j$, $j=1,2,\ldots K$.


\begin{figure*}[!htb]
    \centering
    \includegraphics[width=0.7\textwidth]{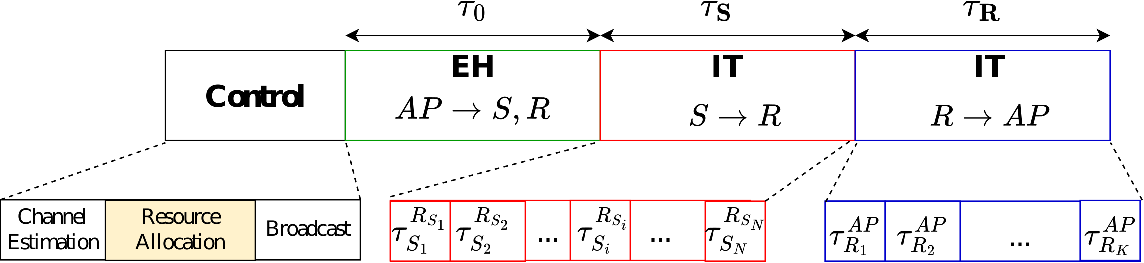}
    \caption{Time Frame}
    \label{fig:timeframe}
\end{figure*}


\subsection{Channel Model}

We consider block-fading channels between $X$ and $Y$ where $X,Y \in \{AP, S_i, R_j\}$, $i = 1,2,\ldots, N$ and $j=1,2,\ldots K$. Channel gains remain constant within each time block but vary independently among different blocks. The DL and UL channel gains between $X$ and $Y$ are denoted by $h_{X}^Y $ and  $g_{X}^Y $, respectively. For instance, $h_{AP}^{S_i}$ is the DL channel gain between $AP$ and $S_i$.  We assume that the AP perfectly knows all the channel gains, similar to the previous works, { e.g., \cite{kazmi2020,kazmi_relay_2021,Onalan2020,xie_age_2023, nguyen_short-packet_2023, Onalan23}. 

 \subsection{Tranmission Model}

 TDMA protocol is used as a medium access control protocol. Fig.~\ref{fig:timeframe} depicts the time division structure. The time block is divided into four time slots, allocated for control, EH, IT of sources, and IT of relays. In a time block, each source $S_i$ needs to send $D_{S_i}>0$ amount of data to the AP  for $i = 1,2,\ldots, N$. 
 
 In the control slot, the sources estimate the channels via pilot signals by the AP and feed the channel gain information back to the AP. Compared to the time and energy cost of IT, the time and energy spent for channel estimation can be considered negligible for a low mobility network \cite{Li2019}. Then, based on the channel gains, the AP runs the proposed algorithms to optimize the resource allocation. As resource allocation problems usually have high complexity, e.g. NP-Hard, the time spent on resource allocation may not be negligible. To complete operations in a time-block within channel coherence time,  low-complexity resource allocation algorithms are required. Finally, the AP broadcasts the information on allocated resources to all the relays and sources in a negligible amount of time.

 In EH slot, the relays and sources harvest energy from the AP during $\tau_0$ amount of time.  The sources and relays possess rechargeable batteries with low energy storage and high self discharge rate, e.g. super-capacitors \cite{Kaus2010}.   Therefore, the harvested energy can be used within the time block but cannot be stored for further use. 



 In the first IT slot,  each source communicates directly to the AP or a selected relay with the harvested energy. Then, in the second IT slot, the selected relays convey the gathered information from sources to the AP.
IT time slots are divided into sub-slots corresponding to each source and relay.  Assume that  the $j^{th}$ relay is selected by the $i^{th}$ source for cooperation. Then, a sub-slot is allocated for the IT from source $S_i$ to relay $R_j$ with duration $\tau_{S_i}^{R_j}$ and another sub-slot is allocated for the IT from $R_j$ to $AP$ with duration $\tau_{R_j}^{AP}$ for $i = 1,2,\ldots, N$ and $j=0,1,2,\ldots K$, where $R_0$ refers to $AP$.  If the $i^{th}$ source prefers direct communication with the AP, information transfer is completed in the first IT block in $\tau_{S_i}^{R_0}$ amount of time and $\tau_{R_0}^{AP}$ is obviously zero.  

Note that  each source can use a single relay for IT. However, the relays can assist multiple sources and  must convey all the information they receive from the sources to the AP.

\subsection{Downlink Energy Harvesting}
The AP is assumed to have an unlimited power source. The transmission power of the AP is assumed to be constant and denoted by $P_{A}$. We assume that $P_{A}$  is large enough to ignore AWGN noise component during the EH phase. 

The relays and sources harvest energy from the AP during $\tau_0$ amount of time.  We apply a practical non-linear EH model \cite{Boshkovska2015}. Thus, the energy harvested at the $i^{th}$ source and $j^{th}$ relay is given by
\begin{equation}
	E_{X} = \frac{ \Psi_{X} - M_{X}\Omega_{X}}{1-\Omega_{X}}\tau_0,
\label{eq:EH_source}
\end{equation}
 where 
\begin{align}
    \Omega_{X} &= \frac{1}{1+exp(a_{X}b_{X})}, \\
    \Psi_{X} &= \frac{M_{X}}{1+\exp(-a_{X}(h_{AP}^{X}P_A-b_{X}))} ,
\end{align}
$a_{X}, b_{X}$, and $M_{X}$ are constant parameters characterizing non-linear behaviour of EH circuit, and $X \in \{S_i, R_j\}$, $i = 1,2,\ldots, N$ and $j=1,2,\ldots K$.


\subsection{Uplink Information Transfer}
We consider continuous power model, where the transmit power $P_{S_i}^{R_j}$ and $P_{R_j}^{AP}$ of $S_i$ and $R_j$, respectively, takes any value below a maximum level $P^{max}$\cite{Onalan2020,kazmi_relay_2021}. $P_{S_i}^{R_j}$ and $P_{R_j}^{AP}$ are also limited by the amount of the harvested energy given by Eq.~(\ref{eq:EH_source}). The required energy for decoding the received message at relays is negligible compared to the energy consumption for IT. Thus, {$P_{S_i}^{R_j} \leq {E_{S_i}}/{\tau_{S_i}^{R_j}} \text{ and } P_{R_j}^{AP} \leq {E_{R_j}}/{\tau_{R_j}^{AP}}$}.


The maximum achievable rate during IT is formulated based on continuous transmission rate model and Shannon's channel capacity formula for AWGN channels. There is no interference during IT  as all sources and relays use separate time slots.  Then, the instantaneous UL transmission rates $T_{S_i}^{R_j}$ from ${S_i}$ to ${R_j}$  and $T_{R_j}^{AP}$ from ${R_j}$ to $AP$ are given by


\begin{align}
    T_{S_i}^{R_j} &= W \log_2 \left( 1 + \frac{P_{S_i}^{R_j}g_{S_i}^{R_j}}{WN_0} \right) \\
    T_{R_j}^{AP} &= W \log_2 \left( 1 +  \frac{P_{R_j}^{AP}g_{R_j}^{AP}}{WN_0} \right).
\label{eq:rateS}
\end{align}

\section{Problem Formulation and Solution Strategy} \label{section:probForm}

The relay selection, scheduling, and power control problem with the objective of minimizing schedule length subject to traffic demand, energy, and power constraints is formulated, similar to the optimization problem with linear EH model in \cite{Onalan2020}, as follows:
\begin{subequations} \label{eq:relSel}
\begin{align}
\min \qquad &  \tau_0 + \sum_{i=1}^N \sum_{j=0}^K \tau_{S_i}^{R_j} + \sum_{j=1}^K \tau_{R_j}^{AP} \label{eq:relSel:obj} \\
\text{s.t.}\qquad & \sum_{j=0}^K b_i^j = 1, \qquad i=1,\ldots,N, \label{eq:relSel:relSel}\\
									& P_{S_i}^{R_j}  \leq P^{max} b_i^j, \quad i=1,\ldots,N, \; j=0,\ldots,K, \label{eq:relSel:maxpow1} \\
									& P_{R_j}^{AP}  \leq  \min \left\{P^{max},P^{max} \sum_{i=1}^N  b_i^j \right\}, \nonumber \\ 
									&\hspace{4.4cm} j=0,\ldots,K, \label{eq:relSel:maxpow2} \\
									& P_{S_i}^{R_j} \tau_{S_i}^{R_j} \leq \frac{ \Psi_{S_i} - M_{S_i}\Omega_{S_i}}{1-\Omega_{S_i}}\tau_0 , \qquad i=1,\ldots,N, \nonumber \\ 
									& \hspace{4.4cm} j=0,\ldots,K, \label{eq:relSel:energy1} \\
									& P_{R_j}^{AP} \tau_{R_j}^{AP} \leq \frac{ \Psi_{R_j} - M_{R_j}\Omega_{R_j}}{1-\Omega_{R_j}}\tau_0 , \qquad j=1,\ldots,K, \label{eq:relSel:energy2}\\
									& \tau_{S_i}^{R_j} W \log_2 \left( 1 + \frac{P_{S_i}^{R_j}g_{S_i}^{R_j}}{WN_0} \right)\geq D_{S_i}b_i^j, \quad \nonumber \\ & \hspace{2.4cm} i=1,\ldots,N, \; j=0,\ldots,K,  \label{eq:relSel:demand1}\\
									& \tau_{R_j}^{AP} W \log_2 \left( 1 + \frac{P_{R_j}^{AP}g_{R_j}^{AP}}{WN_0} \right) \geq \sum_{i=1}^N  D_{S_i}b_i^j, \nonumber \\ & \hspace{4.4cm} j=0,\ldots,K, \label{eq:relSel:demand2}\\							
									&  \tau_0,\tau_{R_j}^{AP}, P_{R_j}^{AP},\tau_{S_i}^{R_j}, P_{S_i}^{R_j} \geq 0, \quad i=1,\ldots,N, \; \nonumber \\ & \hspace{4.6cm} j=1,\ldots,K, \label{eq:relSel:nonneg}\\
									& b_i^j \in \{0,1\},  \quad i=1,\ldots,N, \; j=1,\ldots,K.		\label{eq:relSel:integrality}	
\end{align}
\end{subequations}

The variables of the optimization problem are $\tau_0$, EH duration; $\tau_{S_i}^{R_j}$, IT duration from the $i^{th}$ source to the $j^{th}$ relay;  $\tau_{R_j}^{AP}$, IT duration from the $j^{th}$ relay to the AP; $P_{S_i}^{R_j}$, transmit power of  the $i^{th}$ source when it transmits to the $j^{th}$ relay; $P_{R_j}^{AP}$, transmit power of the $j^{th}$ relay when it transmits to the AP; and $b_i^j$, the relay selection parameter that takes value 1 if the $j^{th}$ relay is selected for the $i^{th}$ source and 0 otherwise, for $i= 1,\ldots,N$ and $j=0,\ldots,K$. 

The objective of the optimization problem is to minimize the total duration for EH and IT. Eq.~(\ref{eq:relSel:relSel}) ensures that one and only one relay is selected for each source, allowing for the possibility of a relay being selected for multiple sources. The maximum allowable transmit power of  the sources and relays are set by Eqs.~(\ref{eq:relSel:maxpow1}) and (\ref{eq:relSel:maxpow2}), respectively.  Eq.~(\ref{eq:relSel:maxpow1}) forces $P_{S_i}^{R_j}$  to be zero when $R_j$ is not selected for $S_i$. Similarly, Eq.~(\ref{eq:relSel:maxpow2}) sets  $P_{R_j}^{AP}$  to zero when $R_j$ is not selected for any source. Eqs.~(\ref{eq:relSel:energy1}) and (\ref{eq:relSel:energy2}) are energy causality constraints limiting $P_{S_i}^{R_j}$ and $P_{R_j}^{AP}$  by the harvested energy amount. Traffic demand constraint  Eq.~(\ref{eq:relSel:demand1}) requires that each source $S_i$ conveys a $D_{S_i}$ amount  of data to the AP or the selected relay. The demand constraint  Eq.~(\ref{eq:relSel:demand2}) guarantees that relays  convey all data they gather from sources to the AP. Lastly, Eq.~(\ref{eq:relSel:nonneg}) and (\ref{eq:relSel:integrality}) represent non-negativity and integrality constraints, respectively.

Problem~\ref{eq:relSel} is a non-convex MINLP as all equations except Eqs.~(\ref{eq:relSel:relSel})-(\ref{eq:relSel:maxpow2}) are non-linear and Eqs. (\ref{eq:relSel:energy1})-(\ref{eq:relSel:demand2}) are non-convex.  The problem is NP-hard as proven in \cite{Onalan2020}.

We follow the solution strategy in \cite{Onalan2020}. The reduced form of Problem~(\ref{eq:relSel}) is formulated for a given relay selection. When the relay selection is predetermined, the total schedule length is minimized while satisfying constraints related to energy causality, demand requirements, and maximum transmit power. The selected relays act as sources since they must convey the gathered messages from sources to the AP. Simultaneously, these chosen relays determine the destination for each source. If no relay is selected for a source, AP is the destination. Thus, a scheduling and power control problem is formulated for WPCN with multiple source-destination pairs and an AP as an energy source. The solution to the reduced problem has been well-studied for the linear EH model. An iterative optimal algorithm has been proposed based on bi-level transformation \cite{Onalan2020}, which is named \textit{POWer constrained Multiple User time minimization Algorithm} (POWMU).  Non-linearity in EH affects the required EH length to satisfy the data demand from sources, thereby altering the optimal solution to the scheduling and power control problem. We modify the POWMU algorithm to accommodate the non-linear EH model considered in this work in Section~\ref{sec:schedule}. 

After calculating the minimum schedule length corresponding to a relay selection, the best relay selection can be obtained by searching over all source-relay combinations.  To find the optimum relay selection,  \cite{Onalan2020} suggests a branch-and-bound algorithm (BBA), which we also adapt to the non-linear EH model to use as a benchmark in the performance comparison and training data in the proposed CNN based solution strategy. BBA branches on relay selection variables and solves the scheduling and power control problem when the selection is completed. To prevent the exhaustive search on all possible combinations, the algorithm uses bounding techniques to prune the BB-tree nodes that cannot provide better feasible points than the ones obtained at the previous nodes. For pruning, we develop lower and upper-bound generation techniques specific to our problem. However, at the worst case, BBA has exponential complexity. Although  two BB-based heuristic approaches with lower runtimes have been proposed in \cite{Onalan2020}, these heuristics still iteratively solve complex mathematical equations and suffer from complexity. To provide a low-complexity solution to the relay selection problem, we propose novel CNN architectures and KD techniques in Section~\ref{sec:relay-learning}.


\section{Scheduling and Power Control under Non-Linear EH} \label{sec:schedule}

This section presents the reduced form of Problem~(\ref{eq:relSel}) for the given relay selection. The selected relays are considered sources since they must send information to the AP by exhausting the harvested energy. Let $L$ be the number of selected relays. Then, the resulting number of sources becomes $N'=N+L$. The selected relays are also the destinations of the sources. Accordingly, the scheduling and power control problem in multiple-source-multiple-destination WPCN under non-linear EH conditions is formulated as follows:
\begin{subequations} \label{eq:problm1APNS}
\begin{align}
\min \qquad &  \tau_0 + \sum_{i=1}^{N'} \tau_{S_i} \label{eq:shceduling_obj}\\
\text{s.t.}\qquad & P_{S_i} \tau_{S_i} \leq \frac{ \Psi_{S_i} - M_{S_i}\Omega_{S_i}}{1-\Omega_{S_i}}\tau_0, \; \quad  i=1,2,\ldots,N', \label{eq:energy1APNS}\\
									& \tau_{S_i} W \log_2 \left( 1 + \frac{P_{S_i}g_{S_i}}{WN_0} \right) \geq D_{S_i} , \; i=1,2,\ldots,N', \label{eq:data1APNS}\\
									& P_{S_i} \leq P^{max}, \qquad \quad  i=1,2,\ldots,N', \label{eq:pmax1APNS}  \\
									&\tau_0, \tau_{S_i}, P_{S_i}  \geq 0, \qquad i=1,2,\ldots,N', \label{eq:nonneg1APNS}
\end{align}
\end{subequations}
where $g_{S_i}$ is the UL channel gain between $S_i$ and its destination; $P_{S_i}$ is the transmit power of $S_i$; and $\tau_{S_i}$ is the IT time from $S_i$ to its destination for $i=1,\ldots, N'$.  The variables of Problem~(\ref{eq:problm1APNS}) are $P_{S_i}, \tau_{S_i}$ and $\tau_0$. Problem~(\ref{eq:problm1APNS}) aims to minimize the total duration for EH and IT. Eqs.~(\ref{eq:energy1APNS}), (\ref{eq:data1APNS}), (\ref{eq:pmax1APNS}), and (\ref{eq:nonneg1APNS}) represent the constraints for energy causality, demand requirement, maximum transmit power, and non-negativity, respectively.

The objective function given in Eq.(\ref{eq:shceduling_obj}) is convex over $\tau_0$ \cite{Onalan2020}. Thus, the optimal $\tau_0$  can be searched via bisection search whose iterations solve $N'$ subproblems calculating the minimum IT length of each source $S_i$ for the given EH length $\tau_0$.  The sub-problems are formulated for $N' = 1$ and fixed $\tau_0 = \overline{\tau_0} $, as follows: 
\begin{subequations} \label{eq:subproblm1APNS}
\begin{align}
\min \qquad &    \tau_{S_1} \label{eq:subprob_obj} \\
\text{s.t.}\qquad & P_{S_1}\tau_{S_1} \leq \frac{ \Psi_{S_1} - M_{S_1}\Omega_{S_1}}{1-\Omega_{S_1}}\overline{\tau_0}, \label{eq:subenergy1APNS}\\
									& \tau_{S_1} W \log_2 \left( 1 + \frac{P_{S_1}g_{S_1}}{WN_0} \right) \geq D_{S_1}, \label{eq:subdata1APNS}\\
									& P_{S_1} \leq P^{max}, \label{eq:subpmax1APNS}  \\
									& \tau_{S_1}, P_{S_1}   \geq 0, \label{eq:subnonneg1APNS}
\end{align}
\end{subequations}
 
 The objective of the subproblems is minimizing $\tau_{S_1}$ for the fixed ${\tau_0}=\overline{\tau_0}$ where Eqs.~(\ref{eq:subenergy1APNS}), (\ref{eq:subdata1APNS}), and (\ref{eq:subpmax1APNS})  represent energy causality, data causality and maximum allowed transmit power constraints, respectively. Problem~(\ref{eq:subproblm1APNS}) is non-convex due to the non-convexity of Eqs.~(\ref{eq:subenergy1APNS}) and (\ref{eq:subdata1APNS}). Next, we provide the solution to Problem~(\ref{eq:subproblm1APNS}).

 \begin{lemma} \label{lemma:solveSubProb} Let $\overline{\tau_{S_1}}(\overline{\tau_0})$ be the optimal solution to Problem~(\ref{eq:subproblm1APNS}) for $\tau_0 = \overline{\tau_0} $.  The optimal IT duration $\overline{\tau_{S_1}}$  is expressed by
\begin{equation}
   \overline{\tau_{S_1}}(\overline{\tau_0}) =  \frac{D_{S_1}}{W\log_2\left( 1+\frac{P^{max}g_{S_1}}{WN_0}\right)} \label{eq:ITwoR_pmax} 
\end{equation}
if
\begin{equation}
 \overline{\tau_0} \geq \ddot{\tau}_0 = \frac{P^{max}\overline{\tau_{S_1}(\overline{\tau_0})}(1-\Omega_{S_1})}{\Psi_{S_1} - M_{S_1}\Omega_{S_1}}.
 \label{eq:EHwoR_pmax}   
\end{equation}
Otherwise, $\overline{\tau_{S_1}}(\overline{\tau_0})$ is derived as a solution of the nonlinear equation
\begin{equation}
\overline{\tau_0}  = \frac{\overline{\tau_{S_1}}(\overline{\tau_0})}{\gamma_{S_1}}\left(2^{\frac{D_{S_1}}{W \overline{\tau_{S_1}}(\overline{\tau_0})}}-1\right)
\label{eq:f_func2},
\end{equation}
where 
\begin{equation}
    \gamma_{S_1} =  \frac{g_{S_1}(\Psi_{S_1} - M_{S_1}\Omega_{S_1})}{WN_0(1-\Omega_{S_1})}.
    \label{eq:gamma}
\end{equation}
\label{lemma:ITforgivenEH}
\end{lemma}

\begin{proof}
    The proof follows the proof of Lemma 3 in \cite{Onalan2020}. Different from \cite{Onalan2020}, Eqs.~(\ref{eq:EHwoR_pmax}) and (\ref{eq:gamma}) are formulated considering non-linear EH conditions.
\end{proof}

\begin{lemma} \label{lemama:bounds}
    The upper and lower bounds for the optimum $\tau_0$ of Problem~(\ref{eq:problm1APNS}) is expressed, respectively, by  
    
    \begin{equation}
        {\tau}_0^{ub} = \max_{i \in \{1,\ldots,N'\}}\frac{P^{max}\overline{\tau_{S_i}}(1-\Omega_{S_i})}{\Psi_{S_i} - M_{S_i}\Omega_{S_i}}
    \end{equation}
    
    and 
    \begin{equation}
        {\tau}_0^{lb}= \max_{i \in \{1,\ldots,N'\}} \frac{D_{S_i}\ln(2)}{W\alpha_{S_i} \gamma_{S_i}} \left(2^{\alpha_{S_i}/\ln(2)}-1\right)
    \end{equation}
    where $\gamma_{S_i} =  \frac{g_{S_i}(\Psi_{S_i} - M_{S_i}\Omega_{S_i})}{WN_0(1-\Omega_{S_i})}$, $ \alpha_{S_i} = \mathbb{L}_0 \left(\frac{\gamma_{S_i}-1}{e}\right) +1$, and $\mathbb{L}_0(.) $ is the Lambert W-function in 0 branch. 
\end{lemma}
\begin{proof}
    The upper bound is the maximum of the solutions of the networks where $S_i$ is the single source and transmitting at $P^{max}$, i.e., Eqs.~(\ref{eq:energy1APNS})-(\ref{eq:pmax1APNS}) hold with equality, for $i = 1,\ldots,N'$. The lower bound is the maximum of the optimal solutions of the networks where $S_i$ is the single source, for $i = 1,\ldots,N'$. The proof follows the proof of Lemma 5 and 6 in \cite{Onalan2020} by updating equations with nonlinear EH formulations.
\end{proof}

\begin{algorithm}[H]
	\caption{\small \textbf{Non-linear EH Power Constrained Multiple User Time Minimization Algorithm (NL-POWMU)}}\label{alg:mrttma}
	\begin{algorithmic}[1]
	    \small\STATE {\textbf{Input: } $P^{max}$, $P_A$, $\gamma_{S_i}$, $\alpha_{S_i}$, $\Psi_{S_i}$, $\Omega_{S_i}$, $M_{S_i}$,$D_{S_i}$, for $i =1,\ldots,N'$}
	     \small\STATE { \textbf{Output: } $\tau_{0}$, ${\tau_{S_i}}$, $P_{S_i}$ for $i =1,\ldots,N'$ }
		\small\STATE Compute $\tau_0^{ub}$  and $\tau_0^{lb}$ by Lemma~\ref{lemama:bounds}
		\small\WHILE {$\tau_0^{ub} - \tau_0^{lb} > 2\epsilon$}		
		\small\STATE $\tau_{0}^{'} = \frac{\tau_0^{ub} + \tau_0^{lb}}{2}$		
		\small\STATE \textbf{if }  $\frac{dg(\overline{\tau_0})}{d(\overline{\tau_0})}\rvert_{\overline{\tau_0} = \tau_{0}^{'}} \geq 0$,  \textbf{then } $\tau_0^{ub} \leftarrow \tau_{0}^{'}$.
		\small\STATE \textbf{if }  $\frac{dg(\overline{\tau_0})}{d(\overline{\tau_0})}\rvert_{\overline{\tau_0} = \tau_{0}^{'}} \leq 0$, \textbf{then }  $\tau_0^{lb} \leftarrow \tau_{0}^{'}$.
		\small\ENDWHILE
		\small\STATE Evaluate and return ${\tau_{0}}= \tau_{0}^{'}$, ${\tau_{S_i}}= \overline{\tau_{S_i}}({\tau'_{0}})$ and {$P_{S_i}$},  for $i=1,\ldots,N'$
	\end{algorithmic}
\end{algorithm}

Algorithm~\ref{alg:mrttma} presents the optimal algorithm NL-POWMU. Initially, the algorithm sets the lower ($\tau_0^{lb}$) and upper ($\tau_0^{ub}$) bounds on $\tau_0$ using Lemma~\ref{lemama:bounds} (Line 3). The algorithm, then, applies the bisection search on $\tau_0$  based on the convexity of the objective function over $\tau_0$ (Lines 4-7). This search aims to iteratively converge to the optimal solution by evaluating the slope of the objective function at the midpoint of an EH time interval and shrinking the interval accordingly. In each iteration, the current solution, $\tau_0 ^{'}$, is set to the mean of the $\tau_0^{lb}$ and $\tau_0^{ub}$ (Line 5).  The objective function is denoted by $g(\overline{\tau_0}) = \overline{\tau_0} + \sum_{i \in \{1,\ldots,N'\}} \overline{\tau_{S_i}}(\overline{\tau_{0}})$. If the first derivative of $g(\overline{\tau_0})$ calculated at $\tau_0 ^{'}$ is non-negative, indicating an increasing trend at that point, the upper bound $\tau_0^{ub}$ is set to $\tau_0 ^{'}$ (Line 6). Conversely, if the derivative is non-positive, indicating a decreasing trend at that point, the lower bound $\tau_0^{lb}$ is set to $\tau_0 ^{'}$ (Line 7). This iterative process continues until the difference between $\tau_0^{ub}$ and $\tau_0^{lb}$ becomes smaller than a predefined small constant $2\epsilon$, where $\epsilon >0$ is used to control the solution's accuracy (Line 4). If the derivative becomes 0, indicating the discovery of the optimal solution, both $\tau_0^{ub}$ and $\tau_0^{lb}$ are set to $\tau_0 ^{'}$,  so the algorithm terminates. The last computed value of $ \tau_0 ^{'} $ is the optimal EH time. Subsequently, the algorithm evaluates and returns  the optimal IT times $\overline{\tau_{S_i}}({\tau'_{0}})$ for $\tau_0 ^{'}$ and power allocation $P_{S_i}$ for $i=1,\ldots, N'$ (Line 8).

\section{Teacher-Student Learning based Relay Selection Framework} \label{sec:relay-learning}


This section presents the deep learning framework for the relay selection problem. The conventional approach to obtain source-relay pairs without deep learning is to resort to an iterative optimization algorithm at the beginning of each time block at the AP, which leads to substantial runtimes. In our deep-learning-based design, the AP can predict optimal relay selections with high accuracy and low complexity, thanks to the well-trained deep-learning architecture based on extensive datasets, enabling the generalization of patterns and relationships in the problem.

The relay selection problem is formulated as a classification problem with $(K+1)$ classes representing $(K+1)$ possible relay selection options for each source as follows:
\begin{equation*}
    \mathcal{P}: \{\mathbf{h},\mathbf{g}\}  \rightarrow \{\mathbf{b}\},
\end{equation*}
where $\mathbf{h}$ $\in \mathbb{R}^{(N+K)\times 1}$ is the input vector whose elements are the DL channel gains 
$h_{AP}^{S_i}$, $h_{AP}^{R_j}$  for $i = 1,2,\ldots, N$ and $j=1,2,\ldots K$;  $\mathbf{g}$ $\in \mathbb{R}^{(K+N(K+1))\times 1}$ is another input vector whose elements are the UL channel gains $g_{S_i}^{R_j}$, $g_{R_j}^{AP}$ for $i=1,2,\ldots, N$ and $j=0,1,2,\ldots K$; $\mathbf{b}$ $\in \mathbb{R}^{(K+1)\times N}$ is the output, where $b_i^j$, the element at $j^{th}$ row and $i^{th}$ column, takes value 1 if $R_j$ is selected for $S_i$ for $j=0,1,2,\ldots, K$ and $i=1,2,\ldots, N$, and 0 otherwise.  For instance, in a 2-source-2-relay network, the case where $R_1$ is selected for $S_1$ and no relay is selected for $S_2$ is mapped to
\begin{equation*}
 \mathbf{b}=\begin{bmatrix}
0 & 1 \\
1 & 0 \\
0 & 0 
\end{bmatrix}. 
\end{equation*}
This representation allows us to design a deep learning architecture to obtain the approximation function $\mathcal{P}$ for a low-complexity solution to the relay selection problem. 

\begin{figure}[!htb]
    \centering
    \subfloat[]{\includegraphics[width=.45\textwidth]{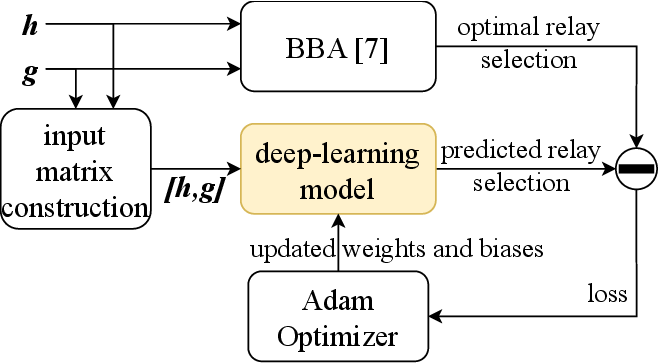} \label{fig:deep-model-train}}
    \subfloat[]{\includegraphics[width=.45\textwidth]{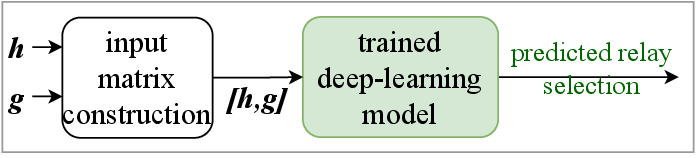}\label{fig:deep-model-pred}}
    \caption{a) Training and  b) Prediction processes of deep learning model}
    \label{fig:deep-model}
\end{figure}

Fig.~\ref{fig:deep-model-train} depicts the training process of the proposed deep learning architecture.
The architecture learns the approximation function $\mathcal{P}$ offline and supervised. The optimal relay selections are obtained via BBA \cite{Onalan2020}. The training process iteratively improves the deep learning model by updating weights and biases with the Adaptive Moment Estimation (Adam) \cite{adam} to minimize the loss function. We employ cross-entropy loss function  
\begin{equation}
 \mathcal{CEL}(\mathbf{\overline{b}},\mathbf{b})= \frac{1}{N}\sum_{i=1}^N \sum_{j=0}^K {b_i^j} \log\left(\frac{\exp{(\overline{b}_i^j)}}{\sum_{j=0}^K \exp{(\overline{b}_i^j)}} \right)
\label{eq:cel}
\end{equation}
between the optimal relay selections $\mathbf{b}$ and predicted ones $\mathbf{\overline{b}}$ to observe how well classification model performs, where $\overline{b}_i^j \in \mathbb{R}$ is the output of the deep learning model corresponding to the predicted value of $b_i^j$ for $j=0,1,2,\ldots, K$ and $i=1,2,\ldots, N$. The cross-entropy loss captures the divergence between the probability distribution of the predicted and optimal relay selections.  

After the completion of the training, the resulting deep learning model with optimized weights and biases is used to predict relay selections online at the beginning of each time block at the AP in a single step, as shown in Fig.~\ref{fig:deep-model-pred}. The prediction step predominantly comprises the multiplication and addition of weights and biases, thereby resulting in a significant reduction in runtime. As a result, the computational complexity becomes the burden of offline training instead of online prediction.

In deep-learning architecture design, we prefer CNNs due to the following advantages \cite{Goodfellow-et-al-2016}:
\begin{enumerate}[label=\roman*)]
    \item \textbf{Sparse interactions:}  In contrast to fully connected DNNs in which every output interacts with every input, CNNs exhibit sparse interactions between the inputs and outputs. This is achieved by using smaller kernels than the input. Thanks to sparsity, CNNs require fewer parameters, which enhances memory and computation efficiency.
    

    \item \textbf{Parameter sharing:} Convolution operation provides parameter sharing by using each kernel element across input positions and enabling a single parameter set to be learned for all locations. Although this does not impact forward propagation runtime, convolution operation substantially enhances memory efficiency and statistical efficacy compared to dense matrix multiplications of traditional DNNs.
    \item \textbf{Spatial relationships:} The parameter sharing enables a single set of learned parameters to be applied to different spatial locations. If one parameter set is useful to extract features in one spatial position, then it can be beneficial at other positions to capture shared common patterns. In the relay selection problem, the relation between DL and UL channel gains in both the source-to-relay and relay-to-AP links needs to be learned.  A pattern resulting in $S_1$  to select $R_j$  in one network realization can be repeated for $S_N$ to select $R_j$  in another network realization. As channel gains of different sources are located at different locations of the input matrix,  observing spatial relations is critical.
\end{enumerate}

We employ 2D convolution blocks to include the larger area in convolution operation to discover similar patterns between the source-relay pairs. Therefore, we transform 1D input vectors to the closest size rectangular matrix where the empty elements are filled by zero padding in the input matrix construction step. For instance, in a 1-source-3 relay network, there are four UL and seven DL channel gains, summing to 11 input parameters. We construct a 3x4 input matrix where the last element is filled with zero.

Initially, we design a deep learning architecture consisting of 2D convolutional blocks and skip connections, which is called \textit{Skip connected Convolutional NETwork} (SC-NET). The details of SC-NET are presented in Section~\ref{sec:sc-net}.
 Then, we propose another deep learning architecture by replacing 2D convolutional blocks with the ones running in multiple flows for memory-constrained applications. Such blocks are called \textit{inception blocks}, so the proposed architecture is named as \textit{SKip INception NETwork} (SKIN-NET), and described in detail in Section~\ref{sec:skin-net}. Finally, we compress the SC-NET by reducing the number of layers and nodes in the architecture for even lower computational complexity. The compressed architecture is called MINI-SC-NET. Obviously, a lower number of parameters limits the learning capacity of the network and results in lower accuracy. To overcome such learning limitations, we propose to apply teacher-student learning where a bigger teacher network, SC-NET, is employed to train the smaller student network, which is called STU-SC-NET, in addition to the truth labels. The details on teacher-student learning are given in Section~\ref{sec:stu-sc-net}.

\begin{figure*}[!thb]
    \centering
    \subfloat[]{\includegraphics[width=.22\textwidth]{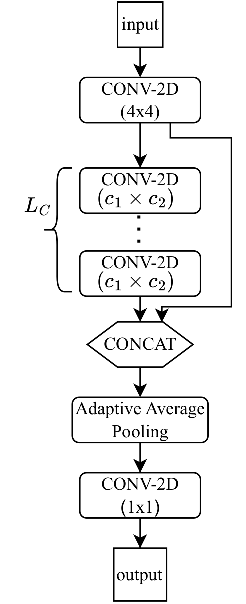}\label{fig:sc-net}} \hfill
    \subfloat[]{\includegraphics[width=.22\textwidth]{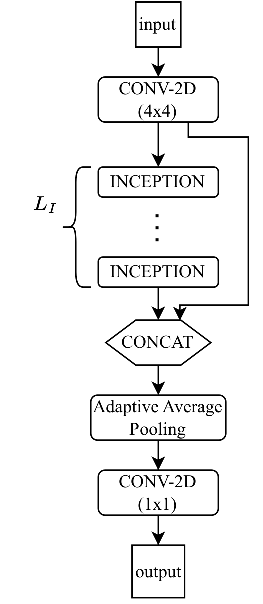}\label{fig:skin-net}} \hfill
    \subfloat[]{\includegraphics[width=.35\textwidth]{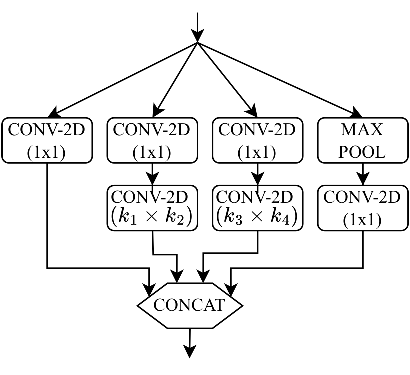}\label{fig:incep}}
    \caption{ a) Proposed Network Architecture with convolution blocks b) Proposed Network Architecture with inception blocks c) Details of inception block}
    \label{fig:network}
\end{figure*}

\subsection{SC-NET Architecture} \label{sec:sc-net}
Fig.~\ref{fig:sc-net}  presents the deep learning architecture of SC-NET. SC-NET  applies a 2D convolution block (CONV-2D) in the first and the last layers with kernel sizes 4x4 and 1x1, respectively. After the first convolution block, there are  $L_C$ 2D convolution blocks with kernel size $c_1$x$c_2$. $L_C$, $c_1$, and $c_2$ are optimized during the training process based on the performance of the architecture on the validation set.  ReLU activation function is applied for all convolutional layers to introduce non-linearity. Further, batch normalization is applied after all convolutional layers for faster and more stable training.

SC-NETs leverage skip connections, which are shortcuts connecting the output of one layer to the input of another layer that may not be immediately adjacent. In deep CNNs, skip connections are helpful in dealing with diluted information or vanishing gradients passing through multiple layers by allowing information to bypass several layers \cite{Tong2017}. The skip connection between two non-adjacent layers can be performed by addition\cite{resnet} or concatenation\cite{huang2017}. We prefer concatenation to combine features from different layers and enhance feature reusability. In SC-NET, the output of the first layer is concatenated to the last layer before the adaptive average pooling.

Following the concatenation step, SC-NET employs adaptive average pooling. Typically, the pooling layer downsamples its inputs, thereby reducing the number of trainable parameters and offering lower complexity \cite{Goodfellow-et-al-2016}. The downsampling also prevents overfitting and leads to generalizable and robust results.  Adaptive average pooling adapts the necessary kernel size to generate an output of the given dimensionality from the given input and performs downsampling by returning the average value of the inputs within the kernel \cite{adaptivepooling}. We prefer adaptive average pooling to align with the desired output vector dimension, concurrently enhancing compression and robustness.

 \subsection{SKIN-NET Architecture} \label{sec:skin-net}
Fig.~\ref{fig:skin-net} presents the deep learning architecture of SKIN-NET. Similar to SC-NET, SKIN-NET also employs 2D convolution blocks with kernel sizes 4x4 and 1x1  in the first and the last layers, adaptive average pooling, and a skip connection from the output of the first layer to the last layer before the adaptive average pooling.  

SKIN-NET replaces $L_C$ convolution blocks of SC-NET with $L_I$ inception blocks. The inception blocks run multiple flows, including convolutional blocks, in parallel to capture more diversified features \cite{inception}. They are preferred especially for memory constrained applications since they can achieve similar performance with lower network parameters.  In our architecture, an inception block includes four parallel flows whose outputs are concatenated at the end. All lines apply 2D convolution with 1x1 kernel size. The fourth line uses max-pooling before the convolution. The second and third lines apply another convolution with kernel size $k_1\times k_2$ and $k_3\times k_4$. Kernel sizes $k_1, k_2, k_3$, and $k_4$ must be optimized based on network size to reach optimum performance. Batch normalization is again applied after all inception layers for faster and more stable training.

\subsection{Teacher-Student Learning} \label{sec:stu-sc-net}

This section proposes teacher-student learning to reduce the complexity of the CNN-based solution without sacrificing optimality. Teacher-student learning transfers knowledge from a well-performing teacher network to a smaller or less powerful student network, using techniques like soft target guidance and distillation \cite{Gou2021}. We use SC-NET as a teacher network to train a smaller student architecture, called STU-SC-NET. To introduce the teacher-student learning concept, we prefer SC-NET over SKIN-NET due to the prolonged training times attributed to inception layers within SKIN-NET. It is worth noting that teacher-student learning is network-agnostic and can be extended to other teacher networks. In this section, we further propose an architecture search algorithm to determine the best student architecture by evaluating the trade-off between complexity and optimality.

A compact architecture can be reached by decreasing $L_C$ and the number of nodes in each convolutional layer. However, standalone training of such a smaller network would be incapable of reaching the accuracy of SC-NET. Therefore, besides the truth labels, we benefit from the expertise of SC-NET to train the smaller student network. Teacher-student learning benefits from both the hard-coded truth labels $\mathbf{b}$, obtained by BBA, and the soft outputs of SC-NET, denoted by $\mathbf{\overline{b}}$, to train STU-SC-NET.  The intuition behind using soft labels from another network is similar to using hard and soft decision coding in the channel coding domain. This approach allows the student model to additionally learn from the teacher's expertise and improve its performance.

The difference between the truth labels $\mathbf{b}$, and the predicted outcomes of STU-SC-NET, $\mathbf{\overline{\overline{b}}}$, where $\overline{\overline{b}}_i^j \in \mathbb{R}$ is the output of student model corresponding to $b_i^j$ for $j=0,1,2,\ldots, K$ and $i=1,2,\ldots, N$, is measured by  cross-entropy loss $\mathcal{CEL}$ by following Eq.~\ref{eq:cel}. Besides, Kullback-Leibler-Divergence,  which is often used as loss function in KD \cite{Gou2021}, is applied between the soft outputs of SC-NET, $\mathbf{\overline{b}}$ and the predicted outcomes of STU-SC-NET $\mathbf{\overline{\overline{b}}}$ as given by
\begin{equation}
 \mathcal{KLD}(\mathbf{\overline{\overline{b}}},\mathbf{\overline{b}})= \frac{1}{N}\sum_{i=1}^N \sum_{j=0}^K \overline{b}_i^j \log\left( \frac{\overline{b}_i^j}{\overline{\overline{b}}_i^j} \right).
 \label{eq:kld}
\end{equation}
Then, the overall training loss function is the weighted average of two loss values and expressed as 
\begin{equation}
 \mathcal{L}(\mathbf{\overline{\overline{b}}},\mathbf{\overline{b}},\mathbf{b}) = \lambda_1 \mathcal{CEL}(\mathbf{\overline{\overline{b}}},\mathbf{b}) +  \lambda_2  \mathcal{KLD}(\mathbf{\overline{\overline{b}}},\mathbf{\overline{b}}) ,
 \label{eq:loss_ts}
\end{equation}
where $\lambda_1$ and $\lambda_2$ are constant weight values.  $\mathcal{L}(\mathbf{\overline{\overline{b}}},\mathbf{\overline{b}},\mathbf{b})$ enables knowledge transfer from teacher SC-NET to student STU-SC-NET by incorporating $\mathcal{KLD}(\mathbf{\overline{\overline{b}}},\mathbf{\overline{b}})$ .


Next, we describe Dichotomous based Architecture Search Algorithm (DASA) for STU-SC-NET, which determines the architecture of STU-SC-NET by initializing the architecture as in the given teacher network and iteratively shrinking it until reaching a certain accuracy level measured by the validation loss.  The architecture of the given teacher network, SC-NET in our case, is determined by its number of layers, $L_C^{SC}$; its set consisting of the number of nodes in each layer, $M^{SC}$; and its total number of trainable parameters, $\Omega^{SC}$. DASA constructs the student architecture by setting the number of layers, the number of nodes in each layer, and the total number of trainable parameters, denoted by $L_C^{STU}$, $M^{STU}$, and  $\Omega^{STU}$, respectively; and outputs the trained student network. An architecture's trainable parameters are based on each layer's layers and nodes, hence, DASA set a target number of trainable parameters for STU-SC-NET and searches for the appropriate $L_C^{STU}$, $M^{STU}$ in each iteration. Since the commonly used grid search method is intractable and computationally expensive for such parameter search \cite{Bao2006,Yang2019}, we apply a method that sequentially adjusts the number of nodes from the last layer to the first layer of the architecture. The number of nodes in each layer is selected from a set of available nodes, denoted by $\mathcal{M}$, rather than a continuous search for computation efficiency. We refer to the algorithm applying this sequential search as \textit{Sequential Parameter Search Algorithm} (Seq-PSA).

\begin{algorithm}
\caption{\small \textbf{ Dichotomous based Architecture Search Algorithm (DASA) for STU-SC-NET}}\label{alg:search_arch}
    \begin{algorithmic}[1]
    \STATE \textbf{Input}:  $\mathbf{b}$, $\mathbf{b'}$, $\mathcal{M}$, $\Omega^{SC}$, $L_C^{SC}$, $M^{SC}$ \label{alg:dasa_line_input}
    \STATE \textbf{Output}: Trained STU-SC-NET \label{alg:dasa_line_output}
    \STATE $L_C^{STU} \leftarrow L_C^{SC}$, $M^{STU} \leftarrow M^{SC}$, $\Omega^{STU}  \leftarrow \Omega^{SC}$ \label{alg:dasa_init}
    \STATE $ub \leftarrow \Omega^{SC}$ and $lb \leftarrow 0$ \label{alg:dasa_bounds}
    \REPEAT  \label{alg:dasa_dicho_begin}
    \STATE $\Omega^{STU} =  [(ub + lb) /2 ]$  \label{alg:dasa_dicho_set}
    \STATE  $L_C^{STU}$, $M^{STU}$, $\Omega^{STU} \leftarrow$ \textbf{Seq-PSA}$(\mathcal{M}$, $\Omega^{STU}, L_C^{STU})$ \label{alg:dasa_update_stu}
    \STATE \textbf{Train} STU-SC-NET \label{alg:dasa_train_stu} and obtain $v^{STU}$
    \STATE \textbf{if} {$v^{STU}$ < $v^{th}$}  \textbf{then} $ub \leftarrow \Omega^{STU} $ \label{alg:dasa_update_ub}
    \STATE \textbf{else} $lb \leftarrow \Omega^{STU}$ \label{alg:dasa_update_lb}
    \UNTIL $ub-lb < \epsilon$ \label{alg:dasa_dicho_end}
    \end{algorithmic}
\end{algorithm}

Algorithm~\ref{alg:search_arch} presents DASA. The truth labels, $\mathbf{b}$; a set of available number of nodes, $\mathcal{M}$; the soft outputs of SC-NET, $\mathbf{b'}$; the architecture of SC-NET with  $L_C^{SC}$, $M^{SC}$, $\Omega^{SC}$,  are the inputs of DASA (Line~\ref{alg:dasa_line_input}). The output of DASA is the trained STU-SC-NET (Line~\ref{alg:dasa_line_output}). The architecture of STU-SC-NET is initialized the same as the teacher network, i.e. the number of layers  $L_C^{STU}$, nodes in each layer $M^{STU}$, and the total number of trainable parameters, $\Omega^{STU}$ of STU-SC-NET are set to $L_C^{SC}$, $M^{SC}$, $\Omega^{SC}$, respectively (Line~\ref{alg:dasa_init}). Dichotomous search (Lines~\ref{alg:dasa_dicho_begin}-\ref{alg:dasa_dicho_end}) is defined over the number of trainable parameters and operates in an interval determined by an upper bound ($ub$)  and a lower bound ($lb$), which are initially set to $\Omega^{SC}$ and 0, respectively (Line~\ref{alg:dasa_bounds}).
In each iteration, Dichotomous search sets $\Omega^{STU}$  to the middle point of the interval (Line~\ref{alg:dasa_dicho_set}), which is the target value to reach with the new architecture. Then, Seq-PSA, presented in Algorithm~\ref{alg:grid},  constructs the new STU-SC-NET architecture by selecting number of layers $L_C^{STU}$ and  nodes in each layer $M^{STU}$,  for given $\mathcal{M}$ and the target value of $\Omega^{STU}$ (Line~\ref{alg:dasa_update_stu}). This STU-SC-NET architecture is trained via teacher-student learning considering the loss function given in Eq.~\ref{eq:loss_ts} (Line~\ref{alg:dasa_train_stu}). The training ends with a validation loss value $v^{STU}$, which is a cross-entropy loss by Eq.\ref{eq:cel}. If $v^{STU}$ is lower than the threshold $v^{th}$, $ub$ is updated as $\Omega^{STU}$ to shrink the architecture further (Line~\ref{alg:dasa_update_ub}). Otherwise, $lb$ is set to $\Omega^{STU}$ to enlarge the architecture to fit the desired validation loss range (Line~\ref{alg:dasa_update_lb}). Iterations of dichotomous search continue until the difference between $ub$ and $lb$ is negligible (Line~\ref{alg:dasa_dicho_end}).


\begin{algorithm}
\caption{\small \textbf{Sequential Parameter Search Algorithm (Seq-PSA) }}\label{alg:grid}
    \begin{algorithmic}[1]
    \STATE \textbf{Input}:  $L_C^{STU}$, $\Omega^{STU}$, and $\mathcal{M}$ \label{alg:grid_line_input} 
    \STATE \textbf{Output}: updated $L_C^{STU}$, $M^{STU}$, and $\Omega^{STU}$ \label{alg:grid_line_output}
    \STATE  $k \leftarrow L_C^{STU}$ and  $\Omega^{STU'} \leftarrow 0$  \label{alg:grid_init_k}
    \REPEAT 
    \STATE  Calculate the corresponding $\Omega^{STU'}_m, \; \forall m \in \mathcal{M}$  at layer $k$ \label{alg:grid_search_m}
    \STATE Pick $m$ with $\min |\Omega^{STU} - \Omega^{STU'}_m|$ for layer $k$   \label{alg:grid_pick_m}
    \STATE Update $M^{STU}$ and $\Omega^{STU'}$ for $m$ nodes at layer $k$ \label{alg:grid_update_omega}
    \STATE $k \leftarrow k-1$  \label{alg:grid_update_k}
    \UNTIL $|\Omega^{STU} - \Omega^{STU'}| < \delta$ and $k>0$ \label{alg:grid_end_cond}
    \STATE $\Omega^{STU} \leftarrow \Omega^{STU'}$  \label{alg:grid_end}
    \end{algorithmic}
\end{algorithm}

Algorithm~\ref{alg:grid} presents Seq-PSA to efficiently determine the next $L_C^{STU}$, $M^{STU}$, and $\Omega^{STU}$ based on the current  $L_C^{STU}$, $M^{STU}$, target $\Omega^{STU}$, and possible number of nodes $\mathcal{M}$.  Seq-PSA adjusts the number of nodes from the last layer to the first layer of the architecture, ensuring convergence to the target $\Omega^{STU}$ with a minimal margin of error. The variable $k$ is used to track the current layer, whose number of nodes are being adjusted in the current iteration, and is initialized as the last layer (Line~\ref{alg:grid_init_k}). The variable $\Omega^{STU'}$ is used to track the current value of $\Omega^{STU}$ through the iterations, and is initialized as zero (Line~\ref{alg:grid_init_k}). For the current layer $k$, Seq-PSA explores all possible values, denoted by $m$, $m \in \mathcal{M}$, and calculates the corresponding number of trainable parameters, denoted by $\Omega^{STU'}_m$ (Line~\ref{alg:grid_search_m}). $m$ providing the closest $\Omega^{STU'}_m$ to $\Omega^{STU}$ is selected for layer $k$ (Line~\ref{alg:grid_pick_m}). Accordingly, $M^{STU}$ and $\Omega^{STU'}$ are updated (Line~\ref{alg:grid_update_omega}). Note that $\mathcal{M}$ always contains 0, and picking $m=0$ means removing the layer from the architecture.  Seq-PSA continues by adjusting the number of nodes at the previous layer (Line~\ref{alg:grid_update_k}) until either all layers are adjusted, or the current $\Omega^{STU'}$ is close enough to target $\Omega^{STU}$ by $delta$ (Line~\ref{alg:grid_end_cond}). At the end, $\Omega^{STU}$ is set to $\Omega^{STU'}$ (Line~\ref{alg:grid_end}).


\section{Performance Evaluation} \label{section:simRes}

The goal of this section is to analyze the performance of the proposed deep learning approaches in terms of training and validation
loss performances, architecture sizes, and runtime complexities, and compare their performance to the optimal solution by BBA, the state-of-art sub-optimal algorithms Opportunistic Relaying (OR) \cite{Chen2015}, One Branch Heuristic (OBH) \cite{Onalan2020}, and the state-of-art deep learning framework called REL-NET \cite{Onalan23}. BBA is the optimal algorithm based on branching on relay selection variables and solving the scheduling and power control problem when the selection is completed, while incorporating pruning techniques using lower and upper-bound generation specific to our problem, as described in more detail in Section~\ref{section:probForm}. OBH applies a branch-and-bound strategy similar to BBA, but instead of all branches, it follows a single branch 
corresponding to selecting $R_j$ for $S_i$  with the maximum $b_i^j \in \mathbb{R}$  as a solution of the relaxed version of Problem~\ref{eq:relSel}. OR considers UL and DL channel gains in both source-to-relay and relay-to-AP links and picks relay $R_j$ with $\argmax_j \min \left(g_{S_i}^{R_j}h_{AP}^{S_i}, g_{R_j}^{AP}h_{AP}^{R_j} \right)$ for each source $S_i$, for $i =1, \ldots, N$. REL-NET is a feed-forward DNN that treats the relay selection problem as a classification problem, where each possible set of source-relay pairing is treated as a distinct class, leading to a total of $(K+1)^N$ classes. This is in contrast to our approach, which formulates the relay selection problem with $(K+1)$ classes. We further use the same compact architecture of STU-SC-NET trained only by the truth labels, called MINI-SC-NET, as a benchmark approach to illustrate the benefit of teacher-student learning.


\subsection{Data Generation} \label{sec:data}

DL and UL channel gains, $\mathbf{h}$  and $\mathbf{g}$, are generated as follows: 
Sources are uniformly distributed in a circular quadrant between radius $3-4$ m. The AP is located at the center of the circle. Relays are positioned between the sources and the AP at a distance of $2$ m from the AP and an equal distance from each other. Large scale {channel} statistics are formulated as $ PL(d) = PL(d_0) - 10\upsilon \log_{10}(d/d_0) + Z$, where $PL(d_0)$ is the free space path loss at unit distance $d_0$ in dB, $d$ is the distance between the transmitter and receiver, $PL(d)$ is the path loss at a distance $d$ in dB, $\upsilon$ is the path-loss exponent, and $Z$ is a zero-mean Gaussian random variable with standard deviation $\sigma_Z$. For small-scale statistics, the Rayleigh fading model is used with scale parameter $\Omega$ set to the mean power level determined by $PL(d)$. 

 $700,000$ independent random network realizations are generated for  each different network size. Generated data is divided into the train, validation, and test sets, with sizes of $689,000$; $10,000$; and $1,000$, respectively. The performance of all the algorithms is measured over the test data set. The simulation results are presented as an average of the performance over $1,000$ independent random network realizations of the test set. 

 We consider the following simulation parameters:  $a_k=150$, $b_k=0.014$, $M_k=0.024$, $W = 1$ MHz,  $N_0 = -90$ dBm,  $PL(d_0)=31.67$ dB, $\sigma_Z= 2$ dB, and $\upsilon= 2$, $P_{AP} = 4$ W, $P^{max} = 10$ mW, $D_{Si} = 50$ bits $\forall i$.  Without loss of generality, $D_{Si}$ values are selected  the same for simpler implementation.

 \begin{figure*}
     \centering
    \subfloat[]{\includegraphics[width=.4\textwidth]{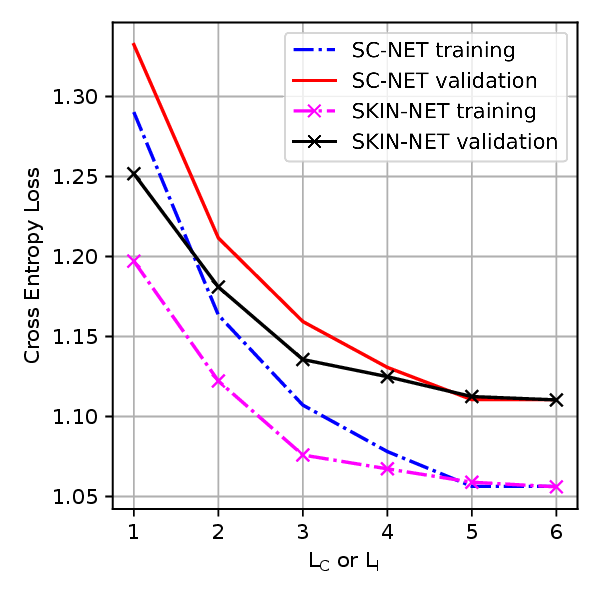}\label{fig:loss_vs_layers}}
    \subfloat[]{\includegraphics[width=.4\textwidth]{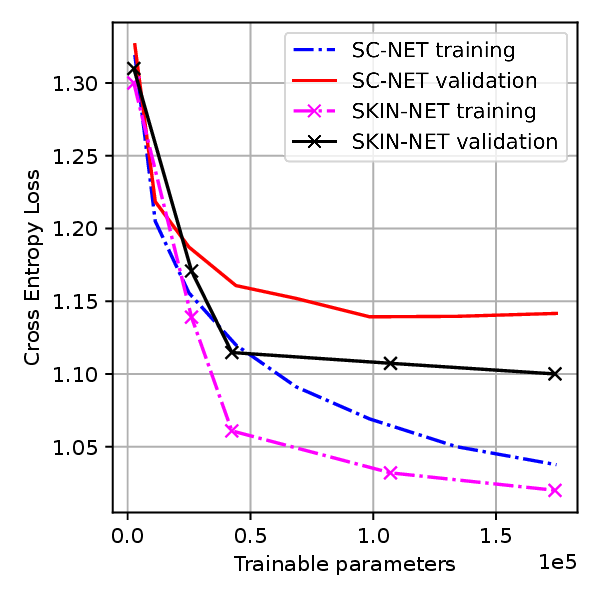}\label{fig:loss_vs_params}}
     \caption{Cross Entropy Loss vs. a) number of layers  b) number of trainable parameters in a 3-source-2-relay network. }
     \label{fig:cross-entropy-loss}
 \end{figure*}

\subsection{Deep Learning Performance}
For the proposed deep learning approaches, this section presents the training and validation loss performances, the proposed architecture sizes, and runtime complexities. The proposed architectures are implemented in PyTorch. Their weights are optimized by Adam optimizer with batch size 128 and learning rate of $10^{-3}$. The hyperparameters such as kernel size and number of layers are fine-tuned via grid search unless otherwise stated. Training is performed on an NVIDIA TITAN XP graphics processing unit (GPU). Testing is performed on  Windows 64-bit OS, 8GB RAM, and Intel i7-7600U dual-core processor. We consider the following algorithm parameters for architecture search: $\mathcal{M}=\{ 64, 32, 24, 16, 8, 2, 0\}$, and $\epsilon =300$.

Fig.~\ref{fig:loss_vs_layers} depicts the training and validation loss for SC-NET and SKIN-NET as the number of layers, $L_C$ or $L_I$,  increases from 1 to 6 while maintaining a fixed number of trainable parameters in a 3-source-2-relay network. We compare SC-NET and SKIN-NET to observe the impact of convolution and inception blocks on the loss performance.  The increasing number of layers improves the cross-entropy loss until $L_C =5$ or $L_I=5$. SKIN-NET provides lower loss values with fewer layers. Lower validation losses with the same complexity underpin the success of inception-based architectures. Note that this graph does not include STU-SC-NET since its loss function evaluates the combination of $\mathcal{CEL}$ and $\mathcal{KLD}$ as given in Eqn.~(\ref{eq:loss_ts}), whereas SC-NET and SKIN-NET employ $\mathcal{CEL}$ only.

Fig.~\ref{fig:loss_vs_params} depicts the training and validation loss for SC-NET and SKIN-NET as the number of trainable parameters increases  in a 3-source-2-relay network. Increasing the number of parameters leads to an enhancement in cross-entropy loss up to 100,000 parameters. SKIN-NET achieves a lower validation and training cross-entropy due to the utilization of inception blocks. These blocks enable the creation of deeper architectures with the same number of trainable parameters compared to SC-NET.

\begin{table}[thb]
\caption{SC-NET Parameters}
\label{tab:sc_net}

\centering
\begin{tabular}{|c|c|c|c|c|c|c|c|}
\hline
N & K & input\_size & $L_C$ & Number of nodes per layer & \begin{tabular}[c]{@{}c@{}}kernel \\      $(c_1 x c_2)$\end{tabular} & Trainable parameters & \begin{tabular}[c]{@{}c@{}}Runtime \\      (ms)\end{tabular} \\ \hline
1 & 2 & 2x4         & 5     & 16,64,64,32,32,16,10      & 2x2                                                                  & 36,457               & 0.11                                                         \\ \hline
2 & 2 & 3x4         & 5     & 16,64,64,32,32,16,10      & 3x2                                                                  & 55,131               & 0.12                                                         \\ \hline
3 & 2 & 4x4         & 5     & 16,64,64,32,32,16,10      & 3x3                                                                  & 86,391               & 0.14                                                         \\ \hline
4 & 2 & 5x4         & 5     & 16,64,64,32,32,16,10      & 4x4                                                                  & 166,543              & 0.21                                                         \\ \hline
5 & 2 & 6x4         & 5     & 16,64,64,32,32,16,10      & 5x4                                                                  & 272,945              & 0.24                                                         \\ \hline
\end{tabular}
\end{table}
\begin{table}[thb]
\caption{SKIN-NET Parameters}
\label{tab:skin_net}

\centering
\begin{tabular}{|c|c|c|c|c|c|c|c|}
\hline
N & K & input\_size & $L_I$ & Number of nodes per layer & \begin{tabular}[c]{@{}c@{}}kernel \\      $(k_1 x k_2)$,\\$(k_3xk_4)$\end{tabular} & Trainable parameters & \begin{tabular}[c]{@{}c@{}}Runtime \\      (ms)\end{tabular} \\ \hline
1 & 2 & 2x4         & 5     & 32,64,64,32,24,16         & 4x4,2x2                                                                        & 24,261               & 0.3                                                          \\ \hline
2 & 2 & 3x4         & 5     & 32,64,64,32,24,16         & 4x4,3x2                                                                        & 25,517               & 0.33                                                         \\ \hline
3 & 2 & 4x4         & 5     & 32,64,64,32,24,16         & 4x4,3x3                                                                        & 27,401               & 0.41                                                         \\ \hline
4 & 2 & 5x4         & 5     & 32,64,64,32,24,16         & 4x4,4x4                                                                        & 31,797               & 0.47                                                         \\ \hline
5 & 2 & 6x4         & 5     & 32,64,64,32,24,16         & 4x4,5x5                                                                        & 37,449               & 0.59                                                         \\ \hline
\end{tabular}
\end{table}

\begin{table}[thb]
\caption{STU-SC-NET Parameters}
\label{tab:small_net}
\centering
\begin{tabular}{|c|c|c|c|c|c|c|c|}
\hline
N & K & input\_size & $L_C$ & Number of nodes per layer & \begin{tabular}[c]{@{}c@{}}kernel \\      $(c_1 x c_2)$\end{tabular} & Trainable parameters & \begin{tabular}[c]{@{}c@{}}Runtime \\      (ms)\end{tabular} \\ \hline
1 & 2 & 2x4         & 1     & 8,8,8,10                    & 2x2                                                                  & 833                  & 0.011                                                       \\ \hline
2 & 2 & 3x4         & 1     & 8,8,8,10                    & 2x2                                                                  & 923                  & 0.011                                                        \\ \hline
3 & 2 & 4x4         & 1     & 8,8,8,10                    & 2x2                                                                  & 1508                  & 0.012                                                       \\ \hline
4 & 2 & 5x4         & 1     & 8,8,8,10                    & 2x2                                                                  & 2507                & 0.013                                                        \\ \hline
5 & 2 & 6x4         & 1     & 8,8,8,10                    & 2x2                                                                  & 6,881               & 0.015                                                        \\ \hline
\end{tabular}
\end{table}

Tables~\ref{tab:sc_net}  and \ref{tab:skin_net}  present the fine-tuned hyperparameters for SC-NET and SKIN-NET, respectively. SKIN-NET possesses 7.35 times less parameters, making it advantageous for memory-constrained applications compared to SC-NET. However, SKIN-NET's parallel flows increase its runtime by three times compared to SC-NET. The runtime tradeoff between SC-NET and SKIN-NET can vary depending on the specific implementation of the algorithms.  In this particular case, the Pytorch implementation does not fully leverage parallel flows in SKIN-NET. Given implementations prioritizing parallel processing, SKIN-NET may exhibit lower runtimes \cite{Pal2019}. Note that the runtimes presented in these tables only consider the testing phase of the relay selection.

Figs.~\ref{fig:dicho_train} and \ref{fig:dicho_val} respectively present the training and validation losses for STU-SC-NET throughout the iterations of DASA given in Algorithm~\ref{alg:search_arch} in a 3-source-2-relay network for a validation cross-entropy loss lower than 1.5. In the first six iterations, DASA shrinks the architecture size. Since the validation loss  exceeds the determined threshold, 1.5, in the sixth iteration, DASA enlarges the architecture for the next step. In the final step, the architecture size reduces one last time and DASA stops. As anticipated, the training and validation losses deteriorate with the decrease in the number of trainable parameters. However, the reduction in the number of parameters implies a decrease in runtime, since the complexity of CNNs is commonly scaled by the number of parameters. DASA is designed to find the optimal balance between the loss value and parameter size. Ultimately, DASA converges to an architecture with 1508 trainable parameters in 8 iterations, achieving a validation cross-entropy loss lower than 1.5.

Table~\ref{tab:small_net} presents the hyperparameters for STU-SC-NET determined by DASA.  The comparison with Tables~\ref{tab:sc_net}  and \ref{tab:skin_net} indicates a significant reduction in the number of trainable parameters and runtimes provided by STU-SC-NET. STU-SC-NET provides up to 16 times lower runtime than SC-NET and up to 29 times less trainable parameters than SKIN-NET. Note that the runtimes presented in the table only consider the testing phase of the relay selection step. 

 \begin{figure*}
     \centering
    \subfloat[]{\includegraphics[width=.35\textwidth]{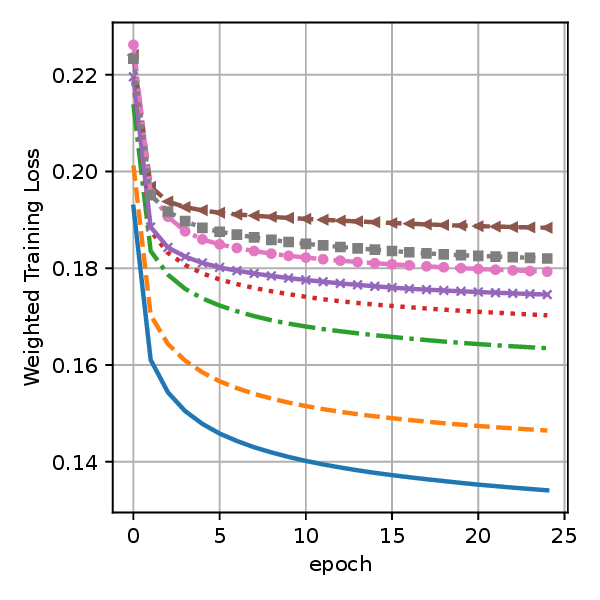}\label{fig:dicho_train}}
    \subfloat[]{\includegraphics[width=.6\textwidth]{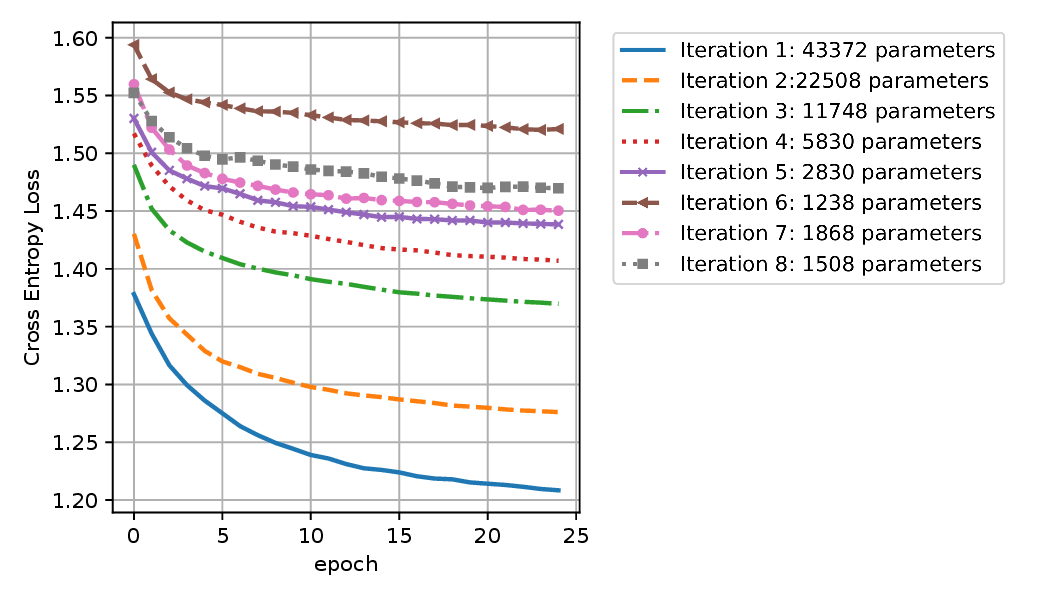}\label{fig:dicho_val}}
     \caption{a) Training and  b) Validation Losses throughout the iterations of DASA given in Algorithm~\ref{alg:search_arch} in a 3-source-2-relay network. }
     \label{fig:dicho-loss}
 \end{figure*}

\subsection{Benchmark Comparison}

This section compares the performances of the proposed deep learning architectures SC-NET, SKIN-NET, and STU-SC-NET, with the aforementioned benchmark algorithms.
 \begin{figure*}
     \centering
     \subfloat[]{\includegraphics[width=.4\textwidth]{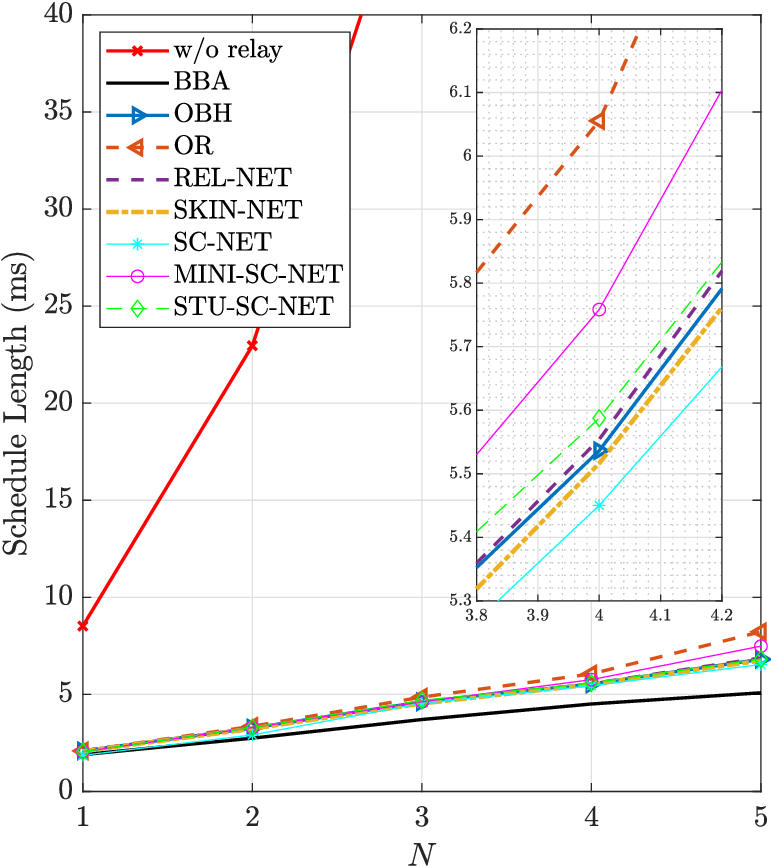}\label{fig:schedule_vs_sources}} 
       \subfloat[]{\includegraphics[width=.4\textwidth]{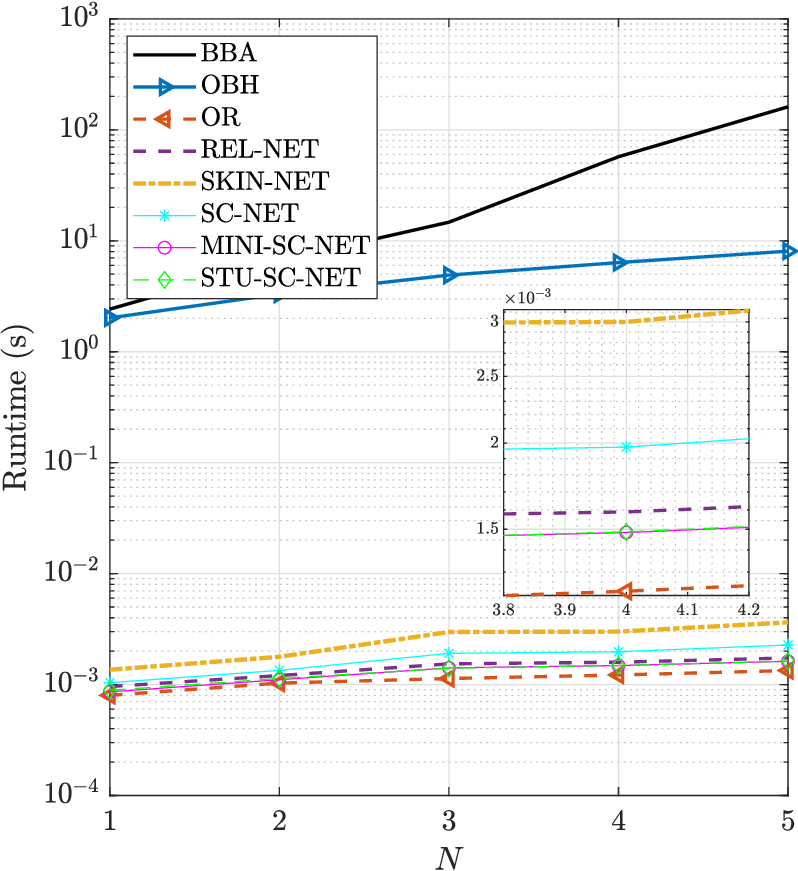}\label{fig:runtime}}
     \caption{a) Schedule length b) Runtime vs. the number of sources in a 2-relay network.}
     \label{fig:schedule}
 \end{figure*}

 Fig.~\ref{fig:schedule_vs_sources} shows the schedule length as the number of sources varies from 1 to 5 in a 2-relay network. Cooperating with relays significantly improves the schedule length, e.g., 95\% for $N=2$. Among the algorithms, BBA stands out as the optimal approach, consistently providing the minimum schedule length. SC-NET outperforms all other sub-optimal approaches, with SKIN-NET closely following. Notably, SKIN-NET achieves a comparable validation loss to SC-NET with fewer parameters, making it advantageous for memory-constrained scenarios. On the other hand, MINI-SC-NET is the worst-performing deep learning-based approach due to its smaller architecture. STU-SC-NET performs up to 9\% better than MINI-SC-NET, indicating the improvement provided by a teacher network during training.  The optimality gap of SC-NET, SKIN-NET, OBH, REL-NET, STU-SC-NET, MINI-SC-NET, and OR are 20\%, 22\% 23\%, 24\%, 27\%, and 34\%, in a 4-source-2-relay network, respectively.

  Fig.~\ref{fig:runtime} shows the runtimes of the proposed and benchmark algorithms as the number of sources varies from 1 to 5 in a 2-relay network. The runtime of MINI-SC-NET is omitted as it overlaps with STU-SC-NET due to the same size architecture. BBA and OBH have three orders of magnitude higher runtimes than all other approaches. OR has the lowest runtime, however, it obtains the highest schedule length. Due to its parallel architecture, SKIN-NET requires more runtime than other deep learning approaches. STU-SC-NET has the lowest runtime among deep learning approaches as it has the smallest architecture.  OBH, SKIN-NET, SC-NET, REL-NET, STU-SC-NET achieve significantly lower runtimes compared to BBA, with reductions up to 20, 44360, 70997, 92916, 99328 times, respectively. 


\section{Conclusion} \label{section:conc}

We study the joint relay selection, scheduling, and power control problem for multiple-source-multiple-relay WPCN with a non-linear EH model. The objective of the problem is to minimize the schedule length, with constraints imposed on data demand, energy causality, and maximum UL transmit power. The problem is non-convex MINLP and NP-hard. We propose a two-step solution strategy. Given the relay selection, a bisection searched-based algorithm solves the scheduling and power control problem. For the low-complexity solution to the remaining relay selection problem, we propose a CNN-based solution, called SC-NET.  We extend the CNN-based solution with inception blocks, called SKIN-NET, to decrease trainable parameter size without sacrificing accuracy for memory-constrained applications. We finally apply teacher-student learning, wherein a compact student network architecture, called STU-SC-NET, is determined through a dichotomous search-based algorithm and subsequently trained with guidance from the larger teacher network, SC-NET. Teacher-student learning supported by an architecture search enables us to decrease the runtime complexity further without sacrificing optimality. We demonstrate via numerical simulations that our proposed approaches exhibit better performance than the state-of-art iterative algorithms. Our proposed deep-learning models are advantageous in terms of runtime since their online prediction step consists of the multiplication and addition of weights and biases. SC-NET stands out for achieving the lowest optimality gap with a relatively low runtime. SKIN-NET can be preferred for memory-constrained applications due to its significantly smaller parameter size compared to SC-NET. Meanwhile, STU-SC-NET is a preferred option for achieving the lowest runtime while maintaining high accuracy.

 The future research direction is to reduce the effort on offline training of deep learning architectures as the training over large datasets consumes considerable energy and time. The reinforcement learning strategies can be applied to avoid generating labels with optimum algorithms; whereas training data points can be smartly selected by active learning strategies to reduce the required amount of data.
 
\bibliographystyle{IEEEtran}
\bibliography{ai_eh_relay,eh_relay,relay_bib}
\end{document}